\documentclass[10pt,twocolumn,letterpaper]{article}

\usepackage{cvpr}      
\usepackage{amsthm}
\usepackage{multirow}
\newtheorem{proposition}{Proposition}
\definecolor{cvprblue}{rgb}{0.21,0.49,0.74}
\usepackage[pagebackref,breaklinks,colorlinks,allcolors=cvprblue]{hyperref}

\def\paperID{*****} 
\def\confName{CVPR}
\def\confYear{2026}

\title{From Structure to Detail: Hierarchical Distillation for Efficient Diffusion Model}

\author{
Hanbo Cheng$^{1,2}$, Peng Wang$^{2}$, Kaixiang Lei$^{2}$, Qi Li$^{2}$, Zhen Zou$^{1,2}$, Pengfei Hu$^{1}$, Jun Du$^{1\dagger}$\\
$^1$University of Science and Technology of China, $^2$ByteDance China \\
}

\begin{document}
\maketitle
\renewcommand{\thefootnote}{} 
\footnotetext{$\dagger$ Corresponding author.}
\begin{abstract}
The inference latency of diffusion models remains a critical barrier to their real-time application. While trajectory-based and distribution-based step distillation methods offer solutions, they present a fundamental trade-off. Trajectory-based methods preserve global structure but act as a ``lossy compressor", sacrificing high-frequency details. Conversely, distribution-based methods can achieve higher fidelity but often suffer from mode collapse and unstable training. This paper recasts them from independent paradigms into synergistic components within our novel Hierarchical Distillation (HD) framework. We leverage trajectory distillation not as a final generator, but to establish a structural ``sketch", providing a near-optimal initialization for the subsequent distribution-based refinement stage. This strategy yields an ideal initial distribution that enhances the ceiling of overall performance. To further improve quality, we introduce and refine the adversarial training process.  We find standard discriminator structures are ineffective at refining an already high-quality generator. To overcome this, we introduce the Adaptive Weighted Discriminator (AWD), tailored for the HD pipeline. By dynamically allocating token weights, AWD focuses on local imperfections, enabling efficient detail refinement. Our approach demonstrates state-of-the-art performance across diverse tasks. On ImageNet $256\times256$, our single-step model achieves an FID of 2.26, rivaling its 250-step teacher. It also achieves promising results on the high-resolution text-to-image MJHQ benchmark, proving its generalizability. Our method establishes a robust new paradigm for high-fidelity, single-step diffusion models.
\end{abstract}    
\section{Introduction}
\label{sec:intro}

\begin{figure}
    \centering
    \includegraphics[width=\linewidth]{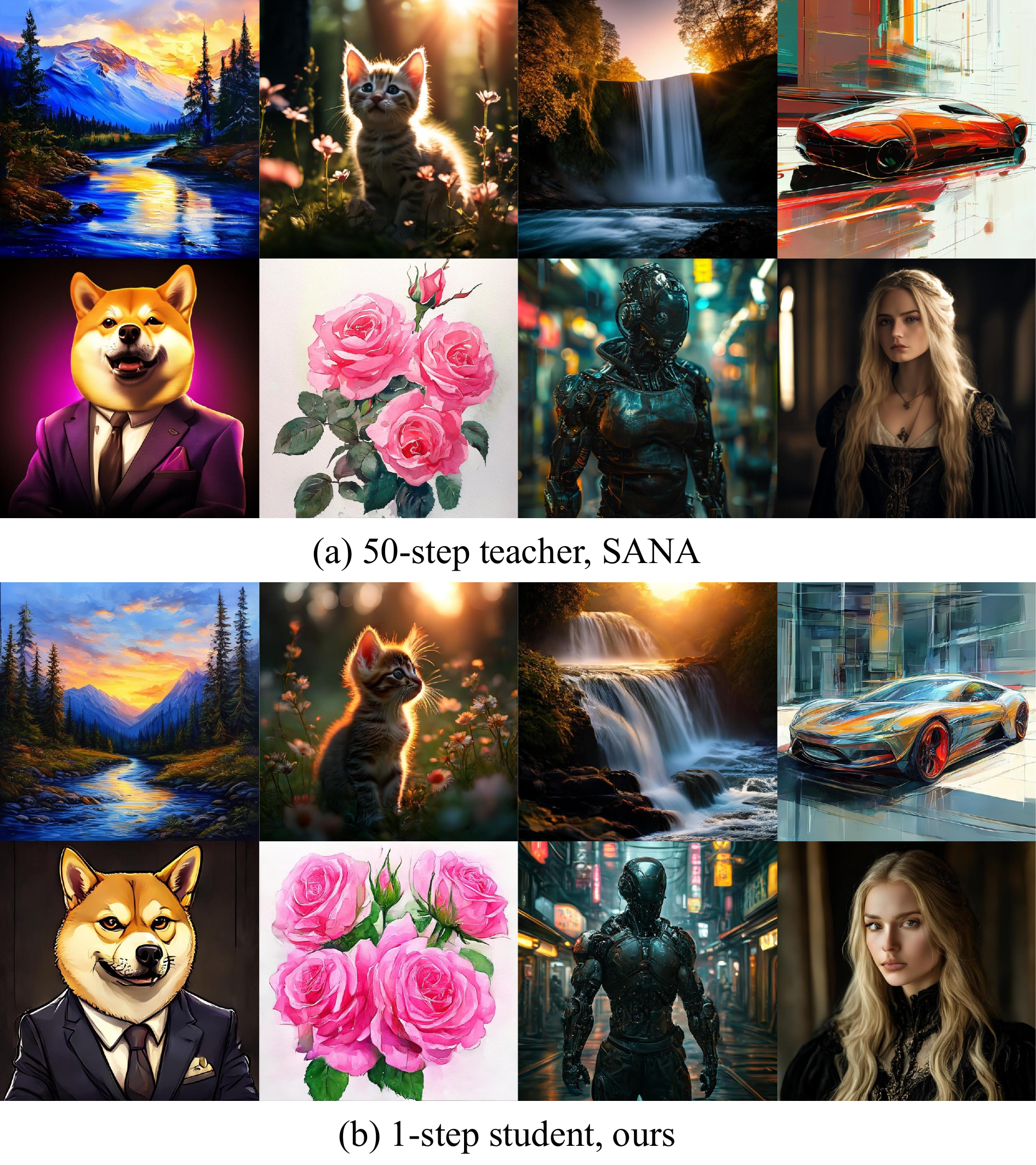}
    \caption{Comparison of generation quality between the 50-step teacher, SANA \cite{SANA}, and our 1-step HD method. Our approach achieves comparable quality to the multi-step teacher.}
    \label{fig:comp_teacher}
\end{figure}

Although diffusion models have revolutionized applications such as visual generation, their widespread, real-time application is severely hindered by high inference latency \cite{score_based, yang2023diffusion, DDPM}. This latency stems from the conventional iterative denoising process which often requires 50 to 100 steps, making it prohibitive for latency-critical applications. Consequently, achieving high-fidelity, single-step or few-step generation has become a crucial research frontier. In response to the aforementioned challenge, numerous studies have explored accelerating diffusion model inference via step distillation \cite{hypersd, meng2023distillationguideddiffusionmodels}. These efforts predominantly fall into two main categories: trajectory-based and distribution-based distillations. Trajectory-based methods, namely the Trajectory Distillation (TD) \cite{CM, LCM, TCM, shortcut, wang2024phased} constrain a student model to ensure its single-step output approximates the complete trajectory generated by a multi-step teacher model solving the Probability Flow Ordinary Differential Equation (PF-ODE). Conversely, distribution-based methods \cite{DMD, fdiv_DMD, LADD, ADD} are designed to directly optimize the student model's output distribution from a single or few steps to match the true data distribution. Both families of methods have yielded significant improvements in inference efficiency.

Nevertheless, when pushed to the limit of single-step generation, these two dominant paradigms reveal their respective inherent and fundamental drawbacks. We propose that TD can be understood as a process of lossy information compression. It forces a student model of limited capacity to reproduce the intricate multi-step temporal dynamics of a teacher model in limited forward passes. This mapping from a high-dimensional trajectory to a low-dimensional output inevitably discards crucial fine-grained details. As a result, although TD methods preserve the structural integrity of the generated content by following the teacher's path, this comes at the cost of final sample fidelity. Conversely, distribution-based approaches \cite{DMD, LADD, TDM, sdxl-lightning} provide the student model with more flexibility, freeing it from the constraints of a potentially suboptimal trajectory and theoretically enabling higher one-step generation quality. However, the optimization process is typically guided by a reverse-KL objective, which exhibits a ``zero-forcing" behavior on the generator's distribution \cite{fdiv_DMD, SiD, ADM}. This property incentivizes the generator to collapse to a few high-density modes, compromising output diversity, especially without a strong initial distribution.

To address these issues mentioned above, we begin by conducting a systematic theoretical and experimental analysis of the performance bottleneck in TD. As shown in Section \ref{sec:Theoretical Grounding} and \ref{sec:toy_experiment}, TDs are highly effective at injecting a ``structural prior" from the teacher's trajectory into the student model, thereby ensuring the macroscopic structure of the generated content. 
However, the information compression inherent in the single-step or few-step mapping inevitably leads to the loss of high-frequency details with limited model capacity. 
This insight reveals a natural complementarity: \textit{the details lost by TD are exactly what distribution-based approaches are designed to refine}. Motivated by this, we propose a novel Hierarchical Distillation (HD) pipeline that first employs TD to generate a structurally sound ``sketch", which provides a well-posed initial distribution for the student model. This pre-trained model then serves as the generator in a subsequent Distribution Matching (DM) stage for refinement, effectively stabilizing the training process. To overcome the potential mode collapse, we further introduce and refine the adversarial distillation technique. However, we observed that when the generator is already well-initialized, conventional discriminators with Global Average Pooling (GAP) are inefficient at providing supervision signals. Therefore, we designed a novel Adaptive Weighted Discriminator (AWD), which dynamically adjusts the spatial weights of the adversarial loss. This allows it to more precisely guide the generator in refining details and enhancing diversity. Our main contributions can be summarized as follows:
\begin{itemize}
    \item Our systematic analysis of Trajectory Distillation (TD) reveals its fundamental nature as a lossy compression process. This explains its inherent trade-off: while effectively retaining global structure, it inevitably fails to preserve fine-grained details.
    \item By reconsidering the respective roles of trajectory-based and distribution-based distillation, we propose a novel Hierarchical Distillation (HD) framework. Complementing this framework, we designed the Adaptive Weighted Discriminator (AWD), a novel adversarial mechanism purpose-built to enhance its final generation quality.
    \item We demonstrate through extensive experiments that our method achieves state-of-the-art performance in single-step generation. Notably, it attains an FID of 2.26 on ImageNet $256\times256$, rivaling the 250-step teacher model, and also delivers highly competitive results on text-to-image synthesis benchmarks.
\end{itemize}

\section{Related Works}

\textbf{Trajectory Distillation (TD)} aims to train a student model to reproduce the multi-step PF-ODE trajectory of a teacher model in just one or a few steps. Early efforts in this domain primarily focused on discrete time steps. For instance, Progressive Distillation \cite{PGD} iteratively distills a two-step inference from the teacher into a single-step output for the student. Consistency Distillation (CD) \cite{CM,iCM} proposes learning a consistent mapping from noisy data at any timestep on the PF-ODE trajectory to the final clean image. Subsequent works such as PCM \cite{PCM}, Shortcut \cite{shortcut}, and CTM \cite{TCM} attempted to reduce the learning difficulty by introducing intermediate anchor points along the ODE trajectory. While improving performance, these methods still inherit the fundamental discretization error. To eliminate discretization errors and pursue maximum trajectory fidelity, research has shifted towards continuous-time trajectory distillation. Methods like \cite{sCM, FACM, sana_sprint} significantly boosted performance by modeling the distillation process in the continuous-time domain. Building on this, Meanflow \cite{Meanflow} further enhanced the stability and effectiveness of continuous TD by modeling the average velocity between two points on the PF-ODE, representing the state-of-the-art in the TD paradigm. However, we argue that all TD methods, including continuous TD, face an even more fundamental limitation in the minimal-step generation context: lossy information compression. Replicating a complex, multi-step trajectory with a single function is inherently lossy. In contrast to prior works that pursue perfect trajectory approximation, our approach embraces this inherent limitation and strategically repurposes the role of TD. Our insight is to use TD not as a final generator, but as a powerful initializer that provides a robust ``structural prior", setting the stage for a subsequent refinement process designed to overcome TD's intrinsic fidelity ceiling.

\textbf{Distribution-based Distillation} paradigm offers an alternative by forgoing the strict replication of a teacher's trajectory. Its core objective is to directly align the data distribution of the student's generated samples with that of the real samples. Early work \cite{ADD, LADD} had already explored using adversarial training to optimize the output distribution of a single-step generator. The Distribution Match Distillation (DMD) algorithm \cite{DMD} further developed this idea by estimating the score fields of both real and generated samples to minimize the KL-divergence between the ground truth distribution and generated distribution. To enhance efficiency, DMD2 \cite{DMD2} adopted adversarial training to eliminate the need for noise-data pairs regularization. Despite opening a new path for high-quality single-step generation, the distribution-based paradigm faces two core challenges: (1) The first challenge is mode collapse: the original DMD typically relies on a reverse-KL objective, which is highly effective at producing high-fidelity samples but is also notoriously prone to mode collapse \cite{DMD, ADM}. To mitigate this, researchers have adopted alternative metrics like JS-divergence \cite{fdiv_DMD} or Fisher divergence \cite{SiD} to improve training stability. (2) The second challenge is that the performance of DMD is highly contingent on the degree of overlap between the student model's initial distribution and the real data distribution. Current solutions are largely heuristic: they either inject noise into the generator's output \cite{DMD} or apply lightweight fine-tuning with LoRA \cite{magicdistillation}. Such strategies aim to broaden the initial distribution through randomness but fail to provide the student with any meaningful prior knowledge of the data manifold. In contrast, our work introduces a structured initialization strategy by repurposing a pre-trained TD model. This approach injects the rich structural knowledge accumulated from the teacher's multi-step inference directly into the student generator. As a result, the student begins the distribution matching from a high-quality, structurally sound position already close to the data manifold, effectively boosting training efficiency and model performance.

\section{Preliminary}



\subsection{Flow Matching with Mean Velocity}
\label{sec:fm_and_mf}

Flow Matching (FM) models~\cite{FlowMatching, FlowMatching2, FlowMatching3} learn a continuous transformation from a prior distribution $p_1$ (e.g., Gaussian noise) to a target data distribution $p_0$. This is achieved by training a neural network $\boldsymbol{v}_\theta(\boldsymbol{x}_t, t)$ to approximate a time-dependent vector field that defines the dynamics of a Probability Flow Ordinary Differential Equation (PF-ODE):
\begin{equation}
    \frac{d\boldsymbol{x}_t}{dt} = \boldsymbol{v}(\boldsymbol{x}_t, t), \quad t \in [0, 1].
    \label{eq:pf_ode}
\end{equation}
The core idea is to define a straight path between a data point $\boldsymbol{x}_0 \sim p_0$ and a noise point $\boldsymbol{x}_1 \sim p_1$, typically via linear interpolation: $\boldsymbol{x}_t = (1-t)\boldsymbol{x}_0 + t \boldsymbol{x}_1$. The ground-truth instantaneous velocity for this path is given by $\boldsymbol{v}(\boldsymbol{x}_t, t) = \boldsymbol{x}_1 - \boldsymbol{x}_0$. The model $\boldsymbol{v}_\theta$ is then trained by minimizing a regression loss:
\begin{equation}
    \mathcal{L}_{\text{FM}}(\theta) = \mathbb{E}_{t, \boldsymbol{x}_0, \boldsymbol{x}_1} \left[ \| \boldsymbol{v}_\theta(\boldsymbol{x}_t, t) - \boldsymbol{v}(\boldsymbol{x}_t, t) \|^2 \right].
    \label{eq:fm_objective}
\end{equation}
Unless otherwise specified, we adopt this linear interpolation path for all experiments in this paper.

After training, generating a sample requires integrating Eq.~\eqref{eq:pf_ode} from $t=1$ to $t=0$, which typically involves a multi-step numerical ODE solver. To enable efficient single-step generation, MeanFlow~\cite{Meanflow} proposed a more effective modeling target: the mean velocity $\boldsymbol{u}(\boldsymbol{x}_t, r, t)$. It is defined as the integral average of the instantaneous velocity over the time interval $[r, t]$:
\begin{equation}
    \boldsymbol{u}(\boldsymbol{x}_t, r, t) \triangleq \frac{1}{t-r} \int_r^t \boldsymbol{v}(\boldsymbol{x}_\tau, \tau) d\tau.
    \label{eq:mean_velocity_def}
\end{equation}
This definition directly provides a path for single-step generation. By setting the interval to $[0, 1]$, we can solve for $\boldsymbol{x}_0$:
\begin{equation}
    \boldsymbol{x}_0 = \boldsymbol{x}_1 - \boldsymbol{u}(\boldsymbol{x}_1, 0, 1).
    \label{eq:meanflow_inference}
\end{equation}
Therefore, if a model $\boldsymbol{u}_\theta$ can accurately predict the mean velocity over the full trajectory, inference reduces to a single function evaluation.

To train such a model $\boldsymbol{u}_\theta(\boldsymbol{x}_t, r, t)$, MeanFlow derives the ground-truth target from the instantaneous velocity. By differentiating Eq.~\eqref{eq:mean_velocity_def} with respect to $t$, one can establish the relation 
\begin{equation}
    \boldsymbol{u}(\boldsymbol{x}_t, r, t) = \boldsymbol{v}(\boldsymbol{x}_t, t) - (t-r) \frac{\partial \boldsymbol{u}(\boldsymbol{x}_t, r, t)}{\partial t}.
    \label{eq:mean_velocity_def2}
\end{equation}
 For stability, the training objective for $\boldsymbol{u}_\theta$ uses a stop-gradient on the target:
\begin{equation}
    \begin{gathered}
        \mathcal{L}_{\text{MF}}(\theta) = \mathbb{E}_{\boldsymbol{x}_0, \boldsymbol{x}_1, r, t} \left[ \| \boldsymbol{u}_\theta(\boldsymbol{x}_t, r, t) - \text{sg}(\boldsymbol{u}_{\text{tgt}}) \|_2^2 \right], \\
        \text{where} \quad \boldsymbol{u}_{\text{tgt}} = \boldsymbol{v}(\boldsymbol{x}_t, t) - (t-r) \frac{\partial \boldsymbol{u}(\boldsymbol{x}_t, r, t)}{\partial t}.
    \end{gathered}
    \label{eq:meanflow_loss}
\end{equation}

\subsection{Distribution Matching Distillation}

Distribution Matching Distillation (DMD)~\cite{DMD} aims to distill a multi-step teacher into single or few-step student models by minimizing the KL-divergence between the generated distribution $p_{\text{fake}}$ and the ground truth distribution $p_{\text{real}}$. However, the probability density functions of these two distributions are inaccessible. To circumvent this, DMD exploits the gradient of the KL-divergence to perform the optimization:
\begin{equation}
    \nabla_\theta \mathcal{D}_{\text{KL}}(p_{\text{fake}} \| p_{\text{real}}) = \mathbb{E}_{\hat{\boldsymbol{x}}_0 \sim p_{\text{fake}}} \left[ (\boldsymbol{s}_{\text{fake}}(\boldsymbol{x}_t) - \boldsymbol{s}_{\text{real}}(\boldsymbol{x}_t)) \nabla_\theta \hat{\boldsymbol{x}_t} \right],
    \label{eq:dmd_gradient_theoretic}
\end{equation}
where $\hat{\boldsymbol{x}}_0 = G_\theta(\boldsymbol{x}_1, t)$ is the clean data predicted by the student model, and $\boldsymbol{s}_{\text{real}}(\boldsymbol{x}) = \nabla_{\boldsymbol{x}} \log p_{\text{real}}(\boldsymbol{x})$, $\boldsymbol{s}_{\text{fake}}(\boldsymbol{x}) = \nabla_{\boldsymbol{x}} \log p_{\text{fake}}(\boldsymbol{x})$ are the score functions of the real and fake distributions, respectively. In practice, DMD employs the frozen teacher model $F_\phi$ to approximate $\boldsymbol{s}_{\text{real}}(\boldsymbol{x})$. Concurrently, it trains a separate ``fake" network $F_\psi$ to dynamically estimate the score of the generated distribution $\boldsymbol{s}_{\text{fake}}(\boldsymbol{x}_t)$. Empirically, the DMD loss is implemented via the stop gradient trick to perform equivalent gradient optimization:

\begin{equation}
    \begin{gathered}
        \mathcal{L}_{\mathbf{DMD}}(\theta)={\operatorname*{\mathbb{E}}}[\|\hat{\boldsymbol{x}}_0-\text{sg}(\hat{\boldsymbol{x}}_0-\boldsymbol{grad})\|_2^2], \\
        \text{where} \quad \boldsymbol{grad} = \frac{t (F_{\psi}(\boldsymbol{x}_t, t) - F_{\phi}(\boldsymbol{x}_t, t)) }{ ||\boldsymbol{G}_\theta(\boldsymbol{x}_t, t) - \boldsymbol{x}_t + tF_{\phi}||}
    \end{gathered}
    \label{eq:dmd_loss}
\end{equation}

\section{Method}
In this section, we present the technical details of our Hierarchical Distillation (HD) framework. We begin in Section \ref{sec:Theoretical Grounding} with a theoretical analysis that unifies mainstream Trajectory Distillation (TD) methods, revealing their shared limitations to motivate our approach. Following this, Section \ref{sec:stage1} details Stage 1 of our pipeline, where a MeanFlow-based TD phase instills a strong structural prior into the student model. Finally, Section \ref{sec:stage2} describes Stage 2, in which we apply distribution matching to this well-initialized model, refining it to achieve high-fidelity results.

\begin{figure*}[t]

   \centering
   \includegraphics[width=0.8\linewidth]{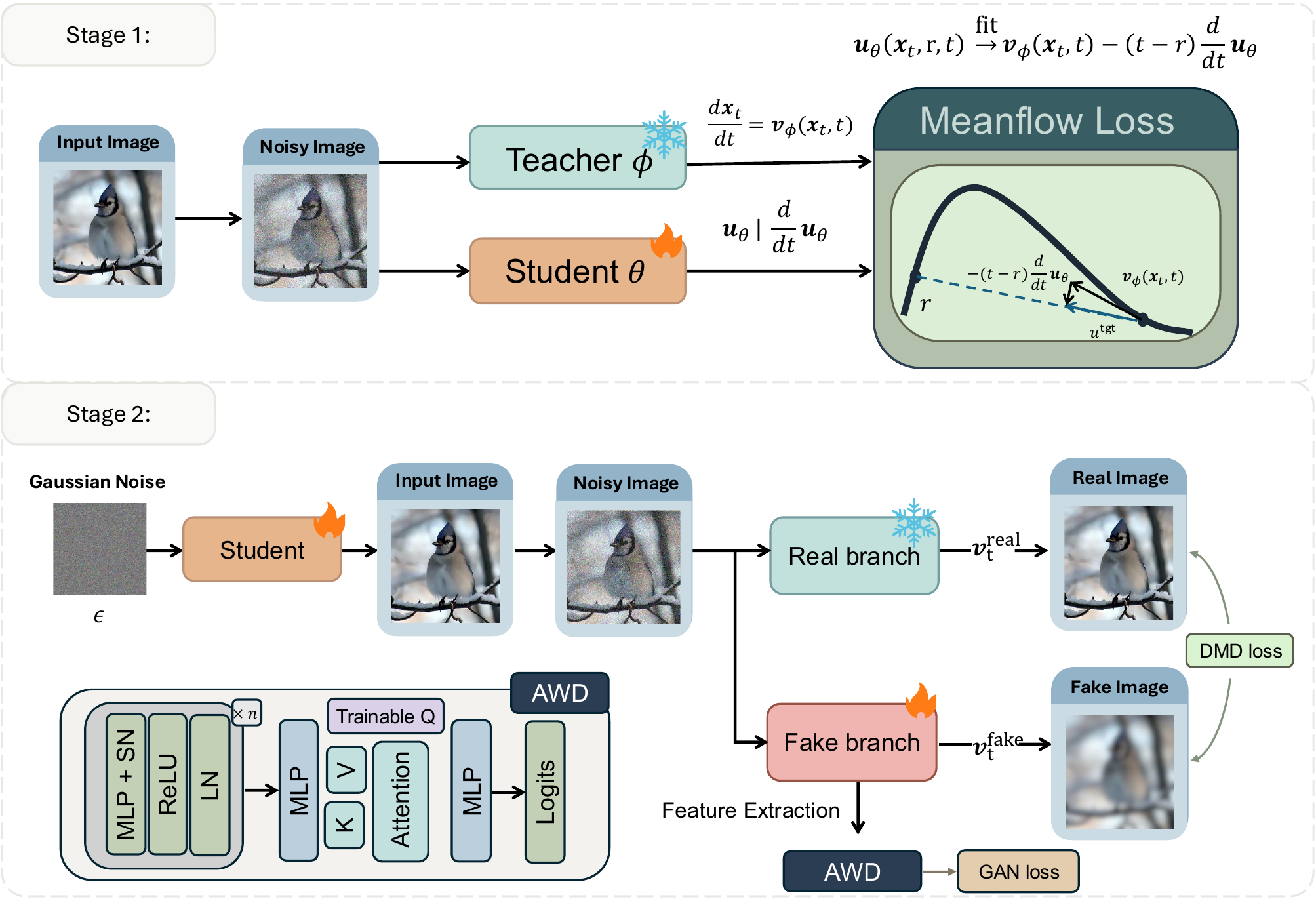}

   \caption{The Hierarchical Distillation (HD) Pipeline. Our method consists of two main stages: (1) Structured Initialization: A MeanFlow-based approach imbues the student with foundational structural information. (2) Distribution Refinement: A second stage restores high-frequency details, employing our Adaptive Weighted Discriminator (AWD) which was specifically designed for the HD framework. The ``SN" and ``LN" refer to spectral norm \cite{SN} and layer norm respectively.}
   \label{fig:pipeline}
\end{figure*}

\subsection{A Unified View of Trajectory Distillation}
\label{sec:Theoretical Grounding}
In this section, we conduct a theoretical analysis to elucidate the modeling target of Trajectory Distillation (TD). Through mathematical derivation, we demonstrate that the objectives of several mainstream TD methods, including Consistency Models (CM/sCM) \cite{CM, sCM} and Progressive Distillation (PGD) \cite{PGD}, can be unified under a common framework of Mean Velocity estimation. From this observation, we identify a common limitation inherent in most TD methods.

\begin{proposition}
Continuous Consistency Models implicitly model the mean velocity over the interval $[0, t]$.
\end{proposition}

\begin{proof}
The core principle of Consistency Models ~\cite{CM} is to enforce consistency for the network output, denoted as $F_\theta(\boldsymbol{x}_t, t)$, along any given PF-ODE trajectory. The differential form of this consistency constraint can be expressed as (see Appendix \ref{appendix:math_proof} for a detailed derivation):
\begin{equation}
    \boldsymbol{x}_t - t \cdot F_{\theta}(\boldsymbol{x}_t, t) = \boldsymbol{x}_{t - \text{d}t} - (t - \text{d}t) \cdot F_{\theta}(\boldsymbol{x}_{t - \text{d}t}, t - \text{d}t),
    \label{eq:scm_discrete_and_continuous}
\end{equation}
where in the limit $\text{d}t \to 0$, this discrete relation yields the differential form:
\begin{equation}
    F_\theta(\boldsymbol{x}_t, t) + t \frac{dF_\theta(\boldsymbol{x}_t, t)}{dt} = \boldsymbol{v}(\boldsymbol{x}_t, t).
    \label{eq:scm_differential_1}
\end{equation}

Recall the relationship between instantaneous and mean velocity from Eq.~\eqref{eq:mean_velocity_def2}. Specifically, for the interval starting at $r=0$, it becomes:
\begin{equation}
     \boldsymbol{u}(\boldsymbol{x}_t, 0, t) + t \frac{\partial \boldsymbol{u}(\boldsymbol{x}_t, 0, t)}{\partial t} = \boldsymbol{v}(\boldsymbol{x}_t, t).
    \label{eq:v_u_relation_r0}
\end{equation}
This reveals that when $\text{d}t \to 0$, the CM network output $F_\theta(\boldsymbol{x}_t, t)$ is implicitly trained to model the mean velocity over the interval $[0, t]$, i.e., $F_\theta(\boldsymbol{x}_t, t) \triangleq \boldsymbol{u}(\boldsymbol{x}_t, 0, t)$.

\end{proof}

\begin{proposition}
Progressive Distillation (PGD) converges to modeling the mean velocity over the full interval $[0, T]$ as the number of distillation steps approaches infinity.
\end{proposition}

\begin{proof}
Progressive Distillation (PGD)~\cite{PGD} is an iterative process that distills a $2^k$-step teacher model into a $2^{k-1}$-step student model over $N$ rounds. In each round $k$, the student model $F_{\theta_k}$ is trained to predict the average of its teacher's ($F_{\theta_{k-1}}$) outputs over two consecutive timesteps.

After $N$ rounds of distillation, the output of the final single-step student model $F_\theta(\boldsymbol{x}_T, T)$ can be expressed as the arithmetic mean of the original multi-step teacher's outputs at $2^N$ discrete timesteps (see Appendix \ref{appendix:math_proof} for a detailed derivation):
\begin{equation}
    F_\theta(\boldsymbol{x}_T, T) = \frac{1}{2^N} \sum_{i=1}^{2^N} F_{\text{teacher}}\left(\boldsymbol{x}_{t_i}, t_i\right),
    \label{eq:pgd_sum}
\end{equation}
where $t_i = i \cdot T/2^N$ are the discrete timesteps. The original teacher model, $F_{\text{teacher}}$, is trained to approximate the instantaneous velocity, i.e., $F_{\text{teacher}}(\boldsymbol{x}_{t_i}, t_i) \approx \boldsymbol{v}(\boldsymbol{x}_{t_i}, t_i)$. As the number of distillation rounds $N \to \infty$, the sum converges to an integral:
\begin{equation}
    \lim_{N \to \infty} \frac{1}{2^N} \sum_{i=1}^{2^N} \boldsymbol{v}(\boldsymbol{x}_{t_i}, t_i) = \frac{1}{T} \int_0^T \boldsymbol{v}(\boldsymbol{x}_\tau, \tau) \, d\tau.
    \label{eq:pgd_integral}
\end{equation}
The right-hand side of Eq.~\eqref{eq:pgd_integral} is, by definition, the mean velocity $\boldsymbol{u}(\boldsymbol{x}_T, 0, T)$ over the entire interval $[0, T]$. This demonstrates that PGD also implicitly attempts to model the mean velocity.
\end{proof}

\textbf{Implications of the Unified View.} Our unifying analysis yields two critical insights that form the theoretical foundation of our proposed method. \textbf{(1)} Our proofs reveal that mainstream TD methods, regardless of their specific formulations, share a common essence: they train a single neural network $F_\theta$ to approximate a dynamic mean velocity function $\boldsymbol{u}(\boldsymbol{x}_t, r, t)$. As an integral over a trajectory segment, this mean velocity function encapsulates rich high-frequency dynamic information from the multi-step teacher's PF-ODE path. Forcing a student model of finite capacity to perfectly replicate a complex function in a single step introduces a fundamental bottleneck from an information-theoretic perspective. This provides a theoretical explanation for why all single-step TD methods inevitably suffer from a loss of fine-grained detail. We provide further empirical validation for this claim in Section~\ref{sec:toy_experiment}. \textbf{(2)} Because all TD methods share the fundamental objective of modeling mean velocity, we selected MeanFlow as their representative implementation. This decision is not merely based on its state-of-the-art performance. More importantly, MeanFlow explicitly and directly models the mean velocity, providing a mathematically elegant and robust implementation.

\subsection{Stage 1: Structured Initialization via TD}
\label{sec:stage1}

As previously discussed, applying Distribution Matching Distillation (DMD) from scratch for single-step generation faces training instability and mode collapse. A primary cause is the lack of overlap between the generation distribution and real data distribution \cite{DMD, magicdistillation}. To address this issue, we introduce a structured initialization stage. We utilize Trajectory Distillation (TD) to efficiently inject the rich structural priors accumulated by a multi-step teacher model into the student. This ensures that before the distribution matching phase even begins, the student model already possesses a strong capability to capture the macroscopic structure and layout of the target distribution. Based on the analysis in Section \ref{sec:Theoretical Grounding}, we adopt MeanFlow~\cite{Meanflow} as the distillation target for our TD stage. Although MeanFlow was originally proposed for training models from scratch, we argue that repurposing it as a distillation framework provides a lower-variance learning signal. When training from scratch, the model learns from random pairings of data and noise, where each sample presents a unique, high-variance target. In contrast, distillation leverages a pretrained teacher that has already converged to a fixed, deterministic mapping from noise to data. This guidance from the teacher ensures that the learning target is consistent during training, thereby lowering the variance of the gradient signal and leading to a more stable and efficient initialization phase.

The pipeline for this stage is illustrated in the upper part of Figure~\ref{fig:pipeline}. In our distillation framework, we replace the ground-truth instantaneous velocity field $\boldsymbol{v}_t$ (originally derived from linear interpolation in standard MeanFlow) with the output of a pre-trained teacher model $F_\phi$. This directly guides the student to learn the teacher's trajectory dynamics. Specifically, we define the instantaneous velocity field using Classifier-Free Guidance (CFG) on the teacher:
\begin{equation}
    v(\boldsymbol{x}_t, t) = (1+w)F_\phi(\boldsymbol{x}_t, t, c) - wF_\phi(\boldsymbol{x}_t, t, \emptyset),
    \label{eq:teacher_velocity}
\end{equation}
where $w$ is the guidance scale, and $c$ and $\emptyset$ represent the conditional and unconditional inputs, respectively. By substituting this teacher-defined velocity field into the MeanFlow training objective (Eq.~\eqref{eq:meanflow_loss}), we construct our distillation loss. The outcome of this stage is a student generator imbued with the teacher's structural prior. Although its minimal-step fidelity is imperfect, it provides a well-posed initialization for the subsequent distribution matching and refinement.



\subsection{Stage 2: Distribution Refinement}
\label{sec:stage2}

Initialized with the Stage 1 model $F_\theta$, the generator $\mathcal{G}_\theta$ then undergoes a second stage of distribution refinement, restoring high-frequency details that are inherently lost when learning solely from the teacher's trajectory.
We employ the DMD-based strategy to align the single-step output distribution of $\mathcal{G}_\theta$ with the real data distribution. Since the initial distribution of $\mathcal{G}_\theta$ already occupies a favorable region on the data manifold with significant overlap with the real distribution, the DMD training proceeds with greater stability and efficiency. Its primary task shifts from ``blind exploration" to ``targeted refinement of details". For the score networks within DMD, although the MeanFlow student itself can predict instantaneous velocity, we still use the pre-trained teacher $F_\phi$ to initialize both the real and fake score branches. This prevents potential error accumulation and provides a more accurate velocity field estimation. The loss function is formulated as Eq.~\eqref{eq:dmd_loss}.

To further stabilize the training and mitigate the risk of mode collapse, we introduce the adversarial training strategy. Instead of discriminating in the high-dimensional pixel space, we introduce a discriminator $\mathcal{D}$ that operates in the feature space of the teacher model $F_\phi$, following \cite{LADD}. 
The overall adversarial loss is composed of a generator and discriminator loss $\mathcal{L}_{\text{adv}}^{\mathcal{G}}$, $\mathcal{L}_{\text{adv}}^{\mathcal{D}}$:

\begin{flalign}
    &\mathcal{L}_{\text{adv}}^{\mathcal{G}} = -\mathbb{E}_{\boldsymbol{x}_1, t, c} \left[ \mathcal{D}(\phi(G_\theta(\boldsymbol{x}_1, t), t, c)) \right], & \label{eq:adv_g} \\
    &\begin{aligned}
        \mathcal{L}_{\text{adv}}^{\mathcal{D}} &= \mathbb{E}_{\boldsymbol{x}_0, t, c} \left[ \text{ReLU}(1 - \mathcal{D}(\phi(\boldsymbol{x}_t, t, c))) \right] \\
        &\quad + \mathbb{E}_{\boldsymbol{x}_1, t, c} \left[ \text{ReLU}(1 + \mathcal{D}(\phi(G_\theta(\boldsymbol{x}_1, t), t, c))) \right],
    \end{aligned} & \label{eq:adv_d}
\end{flalign}
where $\phi(\cdot)$ denotes the feature extraction function, which takes an image, time, and condition as input and returns the intermediate features from the teacher model $F_\phi$. Here, $\boldsymbol{x}_t$ is a noisy real image, while $G_\theta(\boldsymbol{x}_1, t)$ is a generated sample. Finally, the overall loss is formulated as:

\begin{equation}
    \mathcal{L} = \lambda_{1} \mathcal{L}_{\text{DMD}} + \lambda_{2}\mathcal{L}_{\text{adv}}^{\mathcal{G}} + \lambda_{3} \mathcal{L}_{\text{adv}}^{\mathcal{D}}.
    \label{eq:all_loss}
\end{equation}

\paragraph{Adaptive Weighted Discriminator.}
After TD initialization, the student model has captured the overall structure of the target distribution. Imperfections are no longer global but instead manifest as subtle, localized artifacts. This renders conventional discriminators \cite{sana_sprint, accvideo, LADD, DMD2, ADM}, which rely on global average pooling (GAP) largely ineffective. To address this challenge, we design the Adaptive Weighted Discriminator (AWD), as shown at the bottom of Figure \ref{fig:pipeline}. Rather than assigning uniform weights to all tokens, our discriminator employs a learnable query embedding and an attention mechanism to dynamically weigh different tokens on the feature map. Consequently, the discriminator can focus on local regions most likely to contain artifacts, providing more precise and effective gradients to the generator. 

The final student model, trained with this hierarchical framework, is capable of generating images in minimal-step that match the quality of the multi-step teacher while preserving diversity.


\section{Experiment}

\subsection{The Analysis of Bottleneck of Trajectory Distillation}

\label{sec:toy_experiment}

To empirically validate our theoretical claim regarding the information bottleneck in Trajectory Distillation (TD), we designed a 2D toy experiment. Our hypothesis is that distilling a multi-step teacher's trajectory into a single-step student model is fundamentally an act of information compression, and its performance is inherently limited by the student's model capacity.

\paragraph{Experimental Setup.}
We define a mapping task from a prior distribution $p_1$ to a target distribution $p_0$.
\begin{itemize}
    \item The \textbf{prior distribution $p_1$} is a uniform distribution over a rectangle:
    \begin{equation}
        p_1(z) = \mathcal{U}\left(\{ (x, y) \mid x \in [-2, 2], y \in [-2, 0] \}\right).
    \end{equation}
    \item The \textbf{target distribution $p_0$} is the upper semi-circle of a unit circle:
    \begin{equation}
        p_0(z) = \{ (\cos\theta, \sin\theta) \mid \theta \in [0, \pi] \}.
    \end{equation}
    \item The \textbf{generator} is a multi-layer perceptron (MLP) with residual connections.
    \item The \textbf{metric} for evaluation is the Euclidean distance from a generated point $z_{\text{gen}} = (x, y)$ to the unit circle, defined as $d(z_{\text{gen}}) = \left| \sqrt{x^2 + y^2} - 1 \right|$. Lower distance means better performance. 
\end{itemize}

First, we train a teacher model using the standard Flow Matching objective to learn the mapping from $p_1$ to $p_0$. After convergence, we distill it into a single-step student model using MeanFlow as the TD implementation. To investigate the effect of model capacity, we vary the depth of the MLP for both the teacher and student, creating model variants denoted from ``S" to ``XXXL" (see Appendix \ref{appendix:toy_experiment_details} for detailed hyperparameters). We then compare the performance of the teacher using 50 inference steps against the student's single-step generation.

We illustrate the correlation between model capacity and metric in Figure \ref{fig:toy_exp_metric}, and provide a visualization of the generation results of TD student and flow matching teacher (Figure \ref{fig:toy_exp_TD} and Figure \ref{fig:toy_exp_teacher} in Appendix). The results reveal two key findings. (1) As model capacity increases, the performance of the multi-step teacher improves only marginally. In contrast, the performance of the single-step student model improves significantly with increased capacity. This directly supports our hypothesis: when the model capacity is insufficient, forcing a single-step model to compress the teacher's multi-step trajectory inevitably leads to a loss of information and degraded performance. The effectiveness of TD is thus highly dependent on the student model having sufficient capacity spare. (2) Even when the model capacity is increased by over $50\times$, the state-of-the-art TD method (MeanFlow) still fails to perfectly replicate the teacher's trajectory in a single step.

This experiment provides compelling empirical evidence that TD alone is insufficient to achieve optimal single-step generation quality. While the distilled student can capture the general structure of the trajectory, it struggles to preserve the fine-grained details, necessitating a subsequent refinement stage to close the performance gap.

\begin{figure}
    \centering
    \includegraphics[width=0.9\linewidth]{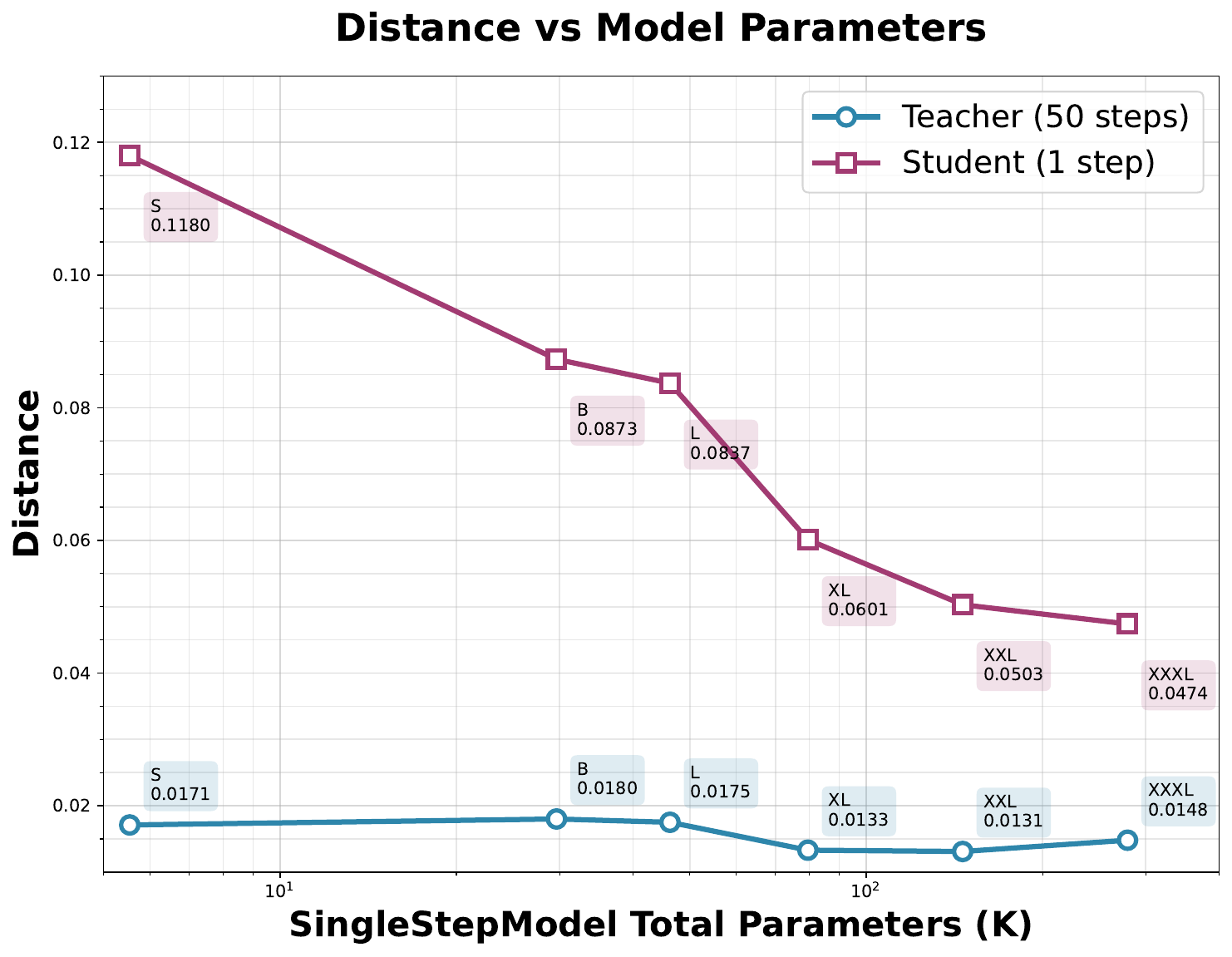}
    \caption{Performance of Trajectory Distillation (TD) vs. Model Size. The upper bound of TD performance increases with the number of model parameters. }
    \label{fig:toy_exp_metric}
\end{figure}

\subsection{Comparison Study}
This section empirically validates the efficacy of our proposed method, demonstrating its ability to achieve state-of-the-art single-step generation quality on both ImageNet 256×256 and text-to-image synthesis. Our primary goal is to drastically reduce inference latency while preserving the high-fidelity output of the original multi-step teacher models. For the ImageNet experiments, we build upon the SiT \cite{SiT} architecture. Following the evaluation procedure of \cite{MAR}, we assess generative performance on 50,000 images. For the text-to-image task, we employ the 0.6B-parameter SANA model \cite{SANA} and evaluate on the MJHQ-30K dataset \cite{MJHQ}. Performance is measured using the following metrics: FID \cite{FID} for perceptual quality and CLIP Score \cite{CLIP} for text-image alignment. In all ImageNet experiments, the Classifier-Free Guidance (CFG) scale is fixed at 1.75 across all timesteps. For the text-to-image task, the CFG scale is fixed at 4.5 following \cite{sana_sprint}. We examine the inference cost by the number of function evaluations (NFE). As detailed in Table \ref{tab:image_net_256}, our method demonstrates remarkable performance. Crucially, in a single generation step, our method achieves an FID of 2.26 on ImageNet, a result that is not only state-of-the-art for one-step models but also comparable to the 250-step teacher model. 
We provide a detailed visual comparison with the multi-step teacher in Figure \ref{fig:qualitative_imagenet} of the Appendix. 
Similarly, in the text-to-image task, as shown in Table \ref{tab:t2i}, our single-step model achieves performance close to the 50-step teacher and outperforms existing step-distillation methods in generation quality. Figure~\ref{fig:qualitative_comparison} presents a qualitative comparison with previous state-of-the-art (SOTA) methods. The results demonstrate that our approach achieves superior performance in preserving both global structural integrity and fine-grained details. Collectively, these results underscore the exceptional capability of our HD method, demonstrating its ability to achieve promising results across a wide range of generative tasks.

\begin{table}[h]
\centering
\caption{Class-conditional generation on ImageNet 256×256. }
\small
\begin{tabular}{lccc}
\toprule
\textbf{Method} & \textbf{Params} & \textbf{NFE} & \textbf{FID ($\downarrow$)} \\
\midrule
\multicolumn{4}{l}{\textbf{\textit{Multi-NFE Baselines}}} \\
SiT-XL/2\cite{SiT}        & 675M & 250×2 & 2.06 \\
DiT-XL/2\cite{DiT}        & 675M & 250×2 & 2.27 \\
SiT-XL/2+REPA\cite{REPA}   & 675M & 250×2 & 1.42 \\
SiT-XL/2 (CFG = 1.75)        & 675M & 250×2 & 2.27 \\
\midrule
\multicolumn{4}{l}{\textbf{\textit{Single-NFE Methods (NFE=1)}}} \\
iCT\cite{iCM}     & 675M & 1     & 34.24 \\
Shortcut\cite{shortcut}    & 675M & 1     & 10.60 \\
IMM\cite{IMM}     & 675M & 1          & 7.77 \\
DMD\cite{DMD}          & 675M & 1     & 6.63 \\
sCM\cite{sCM}          & 675M & 1     & 4.69 \\
FACM\cite{FACM}          & 675M & 1     & 3.68 \\
MeanFlow\cite{Meanflow}     & 676M & 1     & 3.43 \\
Meanflow (distill) & 676M & 1     & 3.62 \\
HD(ours)     & 676M & 1     & 2.26 \\

\bottomrule
\end{tabular}
\label{tab:image_net_256}
\end{table}

\begin{table}[h]
\centering
\caption{Text-to-Image generation on MJHQ-30K dataset.}
\small
\begin{tabular}{lccc}
\toprule
\textbf{Method} & \textbf{NFE} & \textbf{FID ($\downarrow$)} & \textbf{CLIP ($\uparrow$)}\\
\midrule
\multicolumn{4}{l}{\textbf{\textit{Multi-NFE Baselines} } } \\
SDXL\cite{SDXL}        &50×2 & 6.63 & 29.03\\
SD3-medium\cite{SD3-m}   &28×2 & 11.92 & 27.83\\
FLUX-dev\cite{flux}     &50×2 & 10.15 & 27.47\\
SANA \cite{SANA}       &50×2 & 5.76 & 28.67\\
\midrule
\multicolumn{4}{l}{\textbf{\textit{Few-NFE Methods}}} \\
SDXL-LCM \cite{LCM}     & 2     & 18.11 & 27.51\\
PCM\cite{PCM}    & 2     & 14.70 & 27.66\\
SANA-DMD2\cite{DMD2}     & 2   & 9.69  & 27.85\\
SANA-sprint\cite{sana_sprint}   & 2 &  8.41 & 28.18\\
SANA-meanflow\cite{Meanflow} & 2   &  7.53&  28.18\\
SANA-HD(ours)     & 2    & \textbf{7.22} & \textbf{28.32}\\
\midrule
\multicolumn{4}{l}{\textbf{\textit{Single-NFE Methods}}}\\
SDXL-LCM\cite{LCM}     & 1     & 50.51 & 24.45\\
PCM\cite{PCM}    & 1     & 30.11 & 26.47\\
SANA-DMD2\cite{DMD2}     & 1   & 8.45 & 27.95\\
SANA-sprint\cite{sana_sprint}   & 1     & 10.27 & 27.90\\
SANA-meanflow\cite{Meanflow} & 1     & 8.17 & 28.03\\

SANA-HD(ours)     & 1     &\textbf{ 7.47} &  \textbf{28.14}\\

\bottomrule
\end{tabular}
\label{tab:t2i}
\end{table}

\begin{figure}
    \centering
    \includegraphics[width=\linewidth]{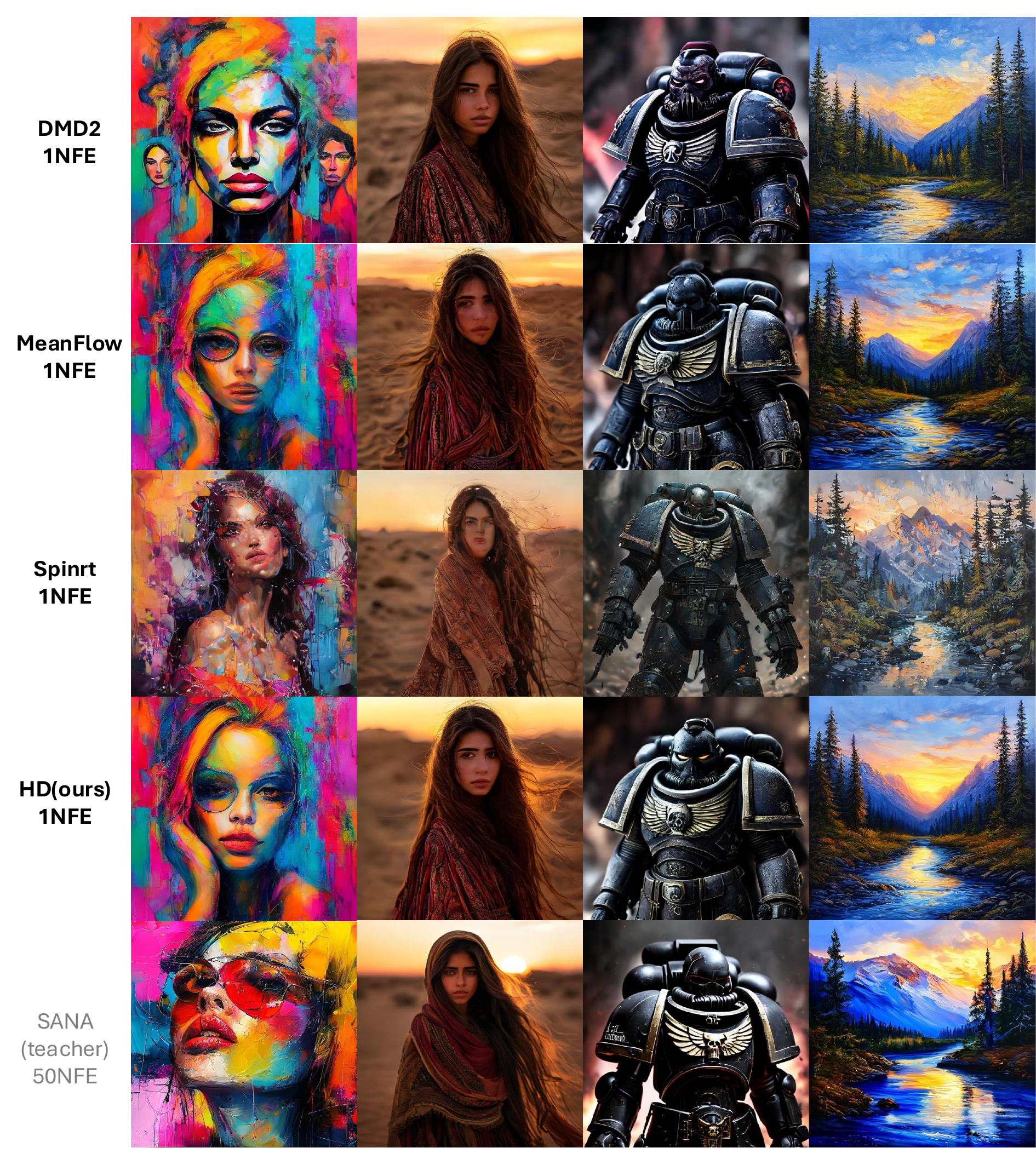}
    \caption{Qualitative comparisons with previous methods based on fully fine-tuned SANA \cite{SANA}.}
    \label{fig:qualitative_comparison}
\end{figure}

\subsection{Ablation Study}
To investigate the effectiveness of our proposed Hierarchical Distillation (HD) framework, we conduct a series of ablation studies on ImageNet to quantify the contribution of each component. As summarized in Table \ref{tab:ablation}, our findings reveal that both stages of the HD framework are crucial for achieving optimal performance. First, we validate the efficacy of the Trajectory Distillation (TD) initialization strategy. When comparing models trained with and without TD initialization while keeping the subsequent DMD and adversarial distillation steps identical, the TD-initialized model demonstrates a significant performance advantage. This confirms that establishing a strong initial trajectory is a critical first step for the student model. While adversarial training is known to mitigate mode collapse of DMD \cite{DMD}, we found the standard adversarial setup to be suboptimal within our HD framework. The ablation study shows that our proposed Attention-Weighted Discriminator (AWD) considerably boosts performance, validating its effectiveness in guiding the student to better align its output distribution with the target.


\begin{table}[t]
\centering
\caption{Ablation study results on ImageNet $256\times256$.}
\small
\begin{tabular}{lcc}
\toprule
\textbf{Method} & \textbf{NFE} & \textbf{FID ($\downarrow$)} \\
\midrule
SiT-XL/2 & 250×2 & 2.27 \\
\midrule
Meanflow (distillation) & 1 & 3.62 \\
DMD & 1 & 6.63 \\
DMD + GAN (GAP) & 1 & 5.49 \\
Meanflow + DMD & 1 & 3.39 \\
Meanflow + DMD + GAN (GAP) & 1 & 3.09 \\
Meanflow + DMD + GAN (AWD) & 1 & \textbf{2.26} \\
\bottomrule
\end{tabular}
\label{tab:ablation}
\end{table}

\section{Conclusion}
In this paper, we present an unified theoretical formulation for Trajectory Distillation (TD), identifying a shared ``average velocity" modeling objective that leads to an information-theoretic bottleneck. This analysis reveals why TD excels at global structures but fundamentally struggles with fine-grained details. Motivated by this insight, we propose a novel Hierarchical Distillation (HD) framework that synergistically combines TD and Distribution Matching. Our method first leverages TD as a powerful initializer to instill rich structural priors from the teacher model, establishing a well-posed starting point for the student. Subsequently, we refine this strong initial model via distribution matching. To enhance this stage, we introduce a tailored adversarial training process with our proposed Adaptive Weighted Discriminator (AWD). By dynamically focusing on the local artifacts of well-initialized models, it provides more precise guidance for detail refinement. Extensive experiments demonstrate that our single-step student model significantly outperforms existing distillation methods and achieves fidelity comparable to its multi-step teacher. By diagnosing and overcoming the bottleneck of TD, our work presents an effective new paradigm for few-step and even single-step high-fidelity generation.
{
    \small
    \bibliographystyle{ieeenat_fullname}
    \bibliography{main}
}

\clearpage
\setcounter{page}{1}
\maketitlesupplementary

\section*{Equivalence of Mean Velocity Estimation and Trajectory Distillation}
\label{appendix:math_proof}
In this section, we provide a detailed proof demonstrating that mainstream trajectory distillation methods are essentially targeted at mean velocity estimation. We discuss this under ideal circumstances (i.e., the model can perfectly fit the training objectives).

\setcounter{proposition}{0}
\begin{proposition}
Continuous Consistency Models (sCM) implicitly model the mean velocity over the interval $[0, t]$.
\end{proposition}

\begin{proof}
The self-consistency method (sCM) is defined by the consistency property, which states that for any time $t$ and an infinitesimal step $\mathrm{d}t$, the model's prediction of the initial state $x_0$ should be consistent. Let $F_\theta(x_t, t)$ be the neural network that predicts the velocity, such that the estimated initial state is $x_0 \approx x_t - t \cdot F_\theta(x_t, t)$. The consistency constraint is thus formulated as:
\begin{equation}
    \boldsymbol{x}_t - t \cdot F_\theta(\boldsymbol{x}_t, t) = \boldsymbol{x}_{t-\mathrm{d}t} - (t-\mathrm{d}t) \cdot F_\theta(\boldsymbol{x}_{t-\mathrm{d}t}, t-\mathrm{d}t)
    \label{eq:scm_consistency}
\end{equation}
We begin by rearranging Equation \eqref{eq:scm_consistency} to isolate the terms related to the state change $\boldsymbol{x}_t - \boldsymbol{x}_{t-\mathrm{d}t}$:
\begin{equation}
    \boldsymbol{x}_t - \boldsymbol{x}_{t-\mathrm{d}t} = t \cdot F_\theta(\boldsymbol{x}_t, t) - (t-\mathrm{d}t) \cdot F_\theta(\boldsymbol{x}_{t-\mathrm{d}t}, t-\mathrm{d}t)
\end{equation}
We can rewrite the right-hand side as:
\begin{align}
    \begin{split}
    \boldsymbol{x}_t - \boldsymbol{x}_{t-\mathrm{d}t} &= t \left[ F_\theta(\boldsymbol{x}_t, t) - F_\theta(x_{t-\mathrm{d}t}, t-\mathrm{d}t) \right] \\
    &\quad + \mathrm{d}t \cdot F_\theta(\boldsymbol{x}_{t-\mathrm{d}t}, t-\mathrm{d}t)
    \end{split}
\end{align}
Now, we divide the entire equation by $\mathrm{d}t$:
\begin{equation}
    \begin{split}
    \frac{\boldsymbol{x}_t - \boldsymbol{x}_{t-\mathrm{d}t}}{\mathrm{d}t} &= t \cdot \frac{F_\theta(\boldsymbol{x}_t, t) - F_\theta(\boldsymbol{x}_{t-\mathrm{d}t}, t-\mathrm{d}t)}{\mathrm{d}t} \\
    &\quad + F_\theta(x_{t-\mathrm{d}t}, t-\mathrm{d}t)
    \end{split}
\end{equation}
Taking the limit as $\mathrm{d}t \to 0$, the terms become their continuous-time derivatives:
\begin{itemize}
    \item $\lim\limits_{\mathrm{d}t \to 0} \frac{\boldsymbol{x}_t - \boldsymbol{x}_{t-\mathrm{d}t}}{\mathrm{d}t} = \frac{\mathrm{d}x_t}{\mathrm{d}t}$, which is the instantaneous velocity, denoted as $\boldsymbol{v}(\boldsymbol{x}_t, t)$.
    \item $\lim\limits_{\mathrm{d}t \to 0} \frac{F_\theta(\boldsymbol{x}_t, t) - F_\theta(\boldsymbol{x}_{t-\mathrm{d}t}, t-\mathrm{d}t)}{\mathrm{d}t} = \frac{\mathrm{d}}{\mathrm{d}t}F_\theta(\boldsymbol{x}_t, t)$, which is the total derivative of the network output with respect to time.
    \item $\lim\limits_{\mathrm{d}t \to 0} F_\theta(x_{t-\mathrm{d}t}, t-\mathrm{d}t) = F_\theta(\boldsymbol{x}_t, t)$, assuming continuity.
\end{itemize}
Substituting these into the equation, we arrive at the differential form of the consistency constraint:
\begin{equation}
    \boldsymbol{v}(\boldsymbol{x}_t, t) = t \frac{\mathrm{d}}{\mathrm{d}t}F_\theta(\boldsymbol{x}_t, t) + F_\theta(\boldsymbol{x}_t, t)
    \label{eq:scm_differential}
\end{equation}
This can be rewritten as:
\begin{equation}
    \boldsymbol{v}(\boldsymbol{x}_t, t) = \frac{\mathrm{d}}{\mathrm{d}t} \left[ t \cdot F_\theta(\boldsymbol{x}_t, t) \right]
\end{equation}
Integrating both sides from $0$ to $t$ with respect to a dummy variable $\tau$:
\begin{equation}
    \begin{split}
    \int_0^t \boldsymbol{v}(\boldsymbol{x}_\tau, \tau) \, \mathrm{d}\tau &= \int_0^t \frac{\mathrm{d}}{\mathrm{d}\tau} \left[ \tau \cdot F_\theta(\boldsymbol{x}_\tau, \tau) \right] \, \mathrm{d}\tau \\
    &= \left[ \tau \cdot F_\theta(\boldsymbol{x}_\tau, \tau) \right]_0^t \\
    &= t \cdot F_\theta(\boldsymbol{x}_t, t) - 0 \cdot F_\theta(\boldsymbol{x}_0, 0) \\
    &= t \cdot F_\theta(\boldsymbol{x}_t, t)
    \end{split}
\end{equation}
    
Finally, solving for $F_\theta(\boldsymbol{x}_t, t)$, we get:
\begin{equation}
    F_\theta(\boldsymbol{x}_t, t) = \frac{1}{t} \int_0^t \boldsymbol{v}(\boldsymbol{x}_\tau, \tau) \, \mathrm{d}\tau
    \label{eq:scm_mean_velocity}
\end{equation}
This equation explicitly shows that the quantity $F_\theta(\boldsymbol{x}_t, t)$ modeled by a continuous consistency model is precisely the average (mean) of the instantaneous velocity $\boldsymbol{v}(\boldsymbol{x}_\tau, \tau)$ over the time interval $[0, t]$.

\end{proof}

\begin{proposition}
Progressive Distillation (PGD) converges to modeling the mean velocity over the full interval $[0, T]$ as the number of distillation steps approaches infinity.
\end{proposition}
\begin{proof}
Let $F_{\theta_0} \equiv F_{\text{teacher}}$. The model $F_{\theta_k}$ is a $2^{N-k}$-step model. The core idea of PGD is that the student model at round $k$ is trained to match the average of two consecutive steps of its teacher model from round $k-1$. The distillation recurrence relation can be expressed as:
\begin{equation}
    F_{\theta_k}(\cdot, t_{j \cdot 2^k}) = \frac{1}{2} \left( F_{\theta_{k-1}}(\cdot, t_{j \cdot 2^k}) + F_{\theta_{k-1}}(\cdot, t_{(j-1/2) \cdot 2^k}) \right),
    \label{eq:pgd_recurrence}
\end{equation}
where $t_m = m \cdot T/2^N$ and $j$ is an integer indexing the time intervals ($j = 1, \dots, 2^{N-k}$).

\textbf{Base Case ($k=1$):}
For the first round of distillation, we train a $2^{N-1}$-step model $F_{\theta_1}$ from the $2^N$-step teacher $F_{\theta_0}$. According to the PGD training objective, for any time interval $[t_{(j-1) \cdot 2}, t_{j \cdot 2}]$, the student's output is the average of the teacher's outputs at two steps. This directly gives:
\begin{equation}
    F_{\theta_1}(\cdot, t_{2j}) = \frac{1}{2} \left( F_{\theta_0}(\cdot, t_{2j}) + F_{\theta_0}(\cdot, t_{2j-1}) \right).
\end{equation}
This matches the form of our proposition for $k=1$.

\textbf{Inductive Hypothesis:}
Assume that after $k-1$ rounds of distillation, the proposition holds. That is, the output of the $2^{N-(k-1)}$-step model $F_{\theta_{k-1}}$ is the arithmetic mean of the original teacher's outputs over $2^{k-1}$ corresponding time steps:
\begin{equation}
    F_{\theta_{k-1}}(\cdot, t_{j \cdot 2^{k-1}}) = \frac{1}{2^{k-1}} \sum_{i=1}^{2^{k-1}} F_{\theta_0}\left(\cdot, t_{(j-1) \cdot 2^{k-1} + i}\right).
    \label{eq:inductive_hypothesis}
\end{equation}
\textbf{Inductive Step:}
We now prove that the proposition holds for round $k$. We train the $2^{N-k}$-step model $F_{\theta_k}$ using the recurrence relation from Eq. \eqref{eq:pgd_recurrence}. We substitute our inductive hypothesis (Eq. \eqref{eq:inductive_hypothesis}) into this relation.
(1) The first term, $F_{\theta_{k-1}}(\cdot, t_{j \cdot 2^k})$, corresponds to the time interval starting from $t_{(j-1/2) \cdot 2^k}$. Applying the hypothesis, we get:
\begin{equation}
    \begin{split}
    F_{\theta_{k-1}}(\cdot, t_{j \cdot 2^k}) &= \frac{1}{2^{k-1}} \sum_{i=1}^{2^{k-1}} F_{\theta_0}\left(\cdot, t_{(j \cdot 2^k - 2^{k-1}) + i}\right) \\
    &= \frac{1}{2^{k-1}} \sum_{i=(j-1/2) \cdot 2^k + 1}^{j \cdot 2^k} F_{\theta_0}(\cdot, t_i).
    \end{split}
\end{equation}
(2) The second term, $F_{\theta_{k-1}}(\cdot, t_{(j-1/2) \cdot 2^k})$, corresponds to the time interval starting from $t_{(j-1) \cdot 2^k}$. Applying the hypothesis again:
\begin{equation}
    \begin{split}
    F_{\theta_{k-1}}(\cdot, t_{(j-1/2) \cdot 2^k}) &= \frac{1}{2^{k-1}} \sum_{i=1}^{2^{k-1}} F_{\theta_0}\left(\cdot, t_{(j-1) \cdot 2^k + i}\right) \\
    &= \frac{1}{2^{k-1}} \sum_{i=(j-1) \cdot 2^k + 1}^{(j-1/2) \cdot 2^k} F_{\theta_0}(\cdot, t_i).
    \end{split}
\end{equation}
Substituting these two expressions back into the recurrence relation for $F_{\theta_k}$:
\begin{equation}
    \begin{split}
    F_{\theta_k}(\cdot, t_{j \cdot 2^k}) &= \frac{1}{2} \left( \frac{1}{2^{k-1}} \sum_{i=(j-1/2) \cdot 2^k + 1}^{j \cdot 2^k} F_{\theta_0}(\cdot, t_i) \right. \\
    &\quad + \left. \frac{1}{2^{k-1}} \sum_{i=(j-1) \cdot 2^k + 1}^{(j-1/2) \cdot 2^k} F_{\theta_0}(\cdot, t_i) \right) \\
    &= \frac{1}{2 \cdot 2^{k-1}} \left( \sum_{i=(j-1) \cdot 2^k + 1}^{j \cdot 2^k} F_{\theta_0}(\cdot, t_i) \right) \\
    &= \frac{1}{2^k} \sum_{i=(j-1) \cdot 2^k + 1}^{j \cdot 2^k} F_{\theta_0}(\cdot, t_i).
    \end{split}
\end{equation}
    
This result matches the form of the inductive hypothesis for round $k$. Thus, the proposition holds by induction.
Finally, after $N$ rounds of distillation, we obtain the single-step model $F_\theta \equiv F_{\theta_N}$. This model is trained to simulate the entire generation process from time $T$ to $0$. This corresponds to the case where $k=N$ and we are evaluating over the full interval, i.e., $j=1$. Setting $k=N$ and $j=1$ in our proven formula, we have:
\begin{equation}
    F_\theta(\cdot, T) = F_{\theta_N}(\cdot, t_{2^N}) = \frac{1}{2^N} \sum_{i=1}^{2^N} F_{\theta_0}(\cdot, t_i).
\end{equation}
Replacing $F_{\theta_0}$ with $F_{\text{teacher}}$ and reintroducing the state dependence $\boldsymbol{x}_{t_i}$, we arrive at the statement of the proposition:
\begin{equation}
    F_\theta(\boldsymbol{x}_T, T) = \frac{1}{2^N} \sum_{i=1}^{2^N} F_{\text{teacher}}(\boldsymbol{x}_{t_i}, t_i).
\end{equation}

The original teacher model, $F_{\text{teacher}}$, is trained to approximate the instantaneous velocity, i.e., $F_{\text{teacher}}(\boldsymbol{x}_{t_i}, t_i) = \boldsymbol{v}(\boldsymbol{x}_{t_i}, t_i)$. As the number of distillation rounds $N \to \infty$, the sum converges to an integral:
\begin{equation}
    \lim_{N \to \infty} \frac{1}{2^N} \sum_{i=1}^{2^N} {v}(\boldsymbol{x}_{t_i}, t_i) = \frac{1}{T} \int_0^T {v}(\boldsymbol{x}_\tau, \tau) \, d\tau.
    \label{eq:pgd_integral_2}
\end{equation}
This derivation shows that the $F_\theta(\boldsymbol{x}_t, t)$ modeled by PGD models the mean velocity over the time interval $[0, T]$.

\end{proof}

\section*{Details of the Experiment in Section \ref{sec:toy_experiment}}
\label{appendix:toy_experiment_details}
The prior and target distributions are illustrated in Figure~\ref{fig:data_dist}. 
Our model is constructed from a stack of Multi-Layer Perceptrons (MLPs) with residual connection, as shown in Figure \ref{fig:data_dist} (b). 
We scale up the model by either increasing the dimensionality of the hidden states or by adding more MLP layers. 
The specific hyperparameter configurations for each model size are detailed in Table~\ref{tab:hyperparams_toy}.

The generation results from the flow matching teacher are shown in Figure~\ref{fig:toy_exp_teacher}. 
The corresponding results from the student models after Trajectory Distillation (TD), across varying model sizes, are presented in Figure~\ref{fig:toy_exp_TD}. 
These observations conform to our analysis in Section~\ref{sec:toy_experiment}, demonstrating that TD performance improves as number of model parameters increases. However, a key finding is that while student performance scales with model capacity, the multi-step teacher's performance has already saturated on this relatively simple task; increasing its parameter count yields no further improvement. 
Furthermore, even when scaling the student model to tens of times the size of the smallest configuration, existing state-of-the-art TD methods still fail to perfectly replicate the teacher's trajectory. Such a result implies that trajectory distillation is essentially an information compression process, which is hard to losslessly replicate teacher's multi-step trajectory.

\begin{figure}
    \centering
    \includegraphics[width=\linewidth]{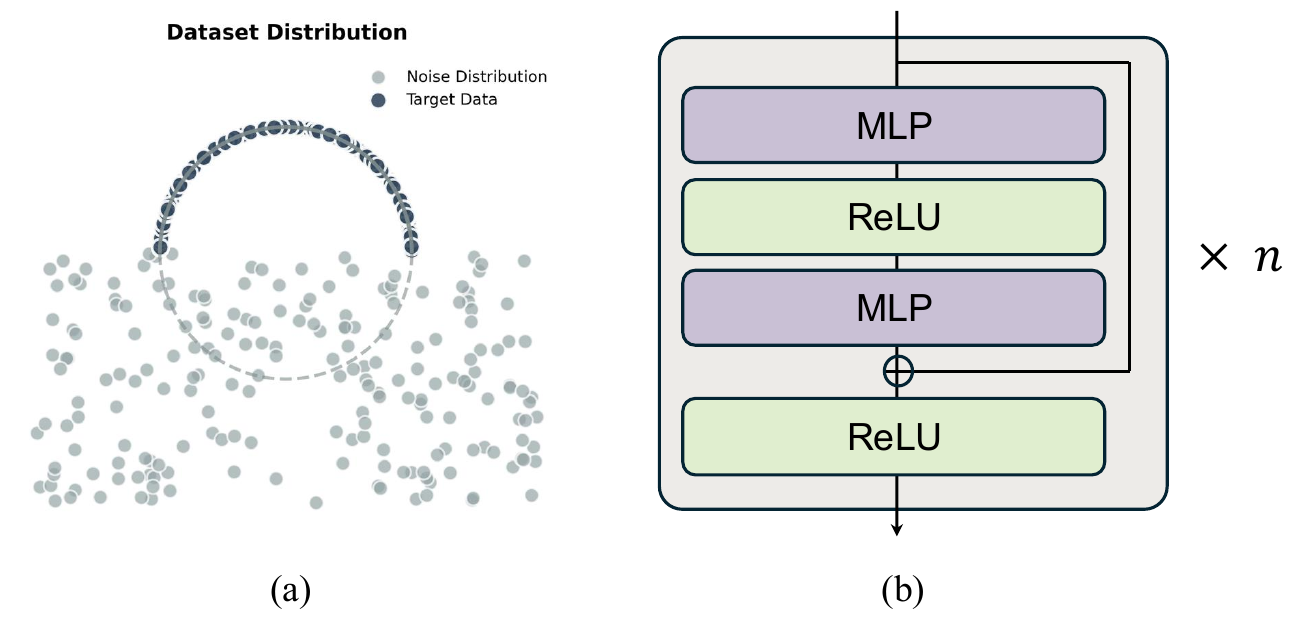}
    \caption{(a) Visualization of the prior (noise) and target distributions. (b) the model architecture used in Section \ref{sec:toy_experiment}}
    \label{fig:data_dist}
\end{figure}

\begin{table}[h]
    \centering
    \caption{}
    \begin{tabular}{lccccc}
        \toprule
        size & params & hidden dim & layers & lr \\
        \midrule
        S & 5.54K & 32 & 1 & 1e-3 \\
        B & 29.63K & 64 & 2 & 1e-3 \\
        L & 46.27K & 64 & 4 & 5e-4 \\
        XL & 79.55K & 64 & 8 & 1e-4 \\
        XXL & 146.11K & 64 & 16 & 1e-4 \\
        XXXL & 279.23K & 64 & 32 & 7e-5 \\
        \bottomrule
    \end{tabular}
    \label{tab:hyperparams_toy}
\end{table}

\section*{Implementation Details}
\label{appendix:implementation_details}

\subsection*{Class Conditional Generation}

For the ImageNet experiments, the training is divided into two stages. 
In Stage 1, we trained the model for 20 epochs. We used the Adam optimizer with a batch size of 256 and a learning rate of $1 \times 10^{-4}$. The optimizer's momentum parameters $(\beta_1, \beta_2)$ were set to $(0.9, 0.95)$. Following the distillation approach of MeanFlow~\cite{Meanflow}, we set the ratio for $r \neq t$ to $100\%$. A Classifier-Free Guidance (CFG) scale of 1.75 was applied across all timesteps. In Stage 2, we fine-tuned the model for 10 epochs. The learning rate for the generator was set to $1 \times 10^{-5}$, while the learning rates for both the fake branch and the discriminator were set to $5 \times 10^{-5}$. For the combined loss function $\mathcal{L}_{\text{s2}}$ (Equation~\ref{eq:all_loss}), the weights for each component were set as follows: $\lambda_1 = 1$, $\lambda_2 = 0.05$, and $\lambda_3 = 0.01$.

\subsection*{Text-to-Image Generation}

For the text-to-image generation experiments, we trained our model on the CC3M dataset~\cite{cc3m}. In Stage 1, we trained the model for 25,000 iterations. Following the setup in~\cite{sana_sprint}, we employed the CAME optimizer~\cite{came} with a learning rate of $1 \times 10^{-5}$ and a batch size of 256. We applied a Classifier-Free Guidance (CFG) scale of 4.5 across all diffusion timesteps. In Stage 2, the model was fine-tuned for an additional 8,000 iterations. We set distinct learning rates for different components: $5 \times 10^{-7}$ for the generator, $5 \times 10^{-6}$ for the fake branch, and $5 \times 10^{-5}$ for the discriminator. The weights for the loss components in $\mathcal{L}_{\text{s2}}$ (Equation~\ref{eq:all_loss}) were configured as $\lambda_1 = 1$, $\lambda_2 = 0.001$, and $\lambda_3 = 0.01$. For the 1-step generation setting, we set the maximum timestep $T$ to 1.0 for all methods. The only exception is SANA-sprint, for which $T$ is set to 1.5708 ($\pi/2$) following its adoption of the trigflow \cite{sCM}. For the 2-step setting, the inference schedule consists of two equally spaced timesteps within the $[0, T]$ interval.

\section*{Limitation and Future work}
\label{appendix:limitation}
Although our work demonstrates powerful minimal-step generation ability, we identify several directions for future research: 1. Towards a Unified End-to-End Framework: Our current two-stage design, which first distills trajectory knowledge before refining the distribution (DMD), was a deliberate choice to maximize model quality. A compelling future direction involves developing a more unified, end-to-end training process that synergistically integrates the principles of TD and DMD. This advancement could further streamline the distillation pipeline and potentially unlock new performance frontiers. 2. Exploring Non-Adversarial Alignment: The adversarial mechanism in our distribution refinement stage was a strategic choice to effectively prevent mode collapse. Although successful, exploring alternative non-adversarial distribution alignment techniques presents a fruitful research direction. Such methods could offer different trade-offs in training stability and computational cost, potentially leading to even more efficient and robust distillation solutions.
3. Scaling to New Models and Domains: Building on the strong results presented, the next step is to validate the scalability and generalizability of our method. We plan to apply the HD framework to larger-scale foundation models and more complex generative domains, such as video synthesis. This will further broaden its impact across the field.

\begin{figure}
    \centering
    \includegraphics[width=\linewidth]{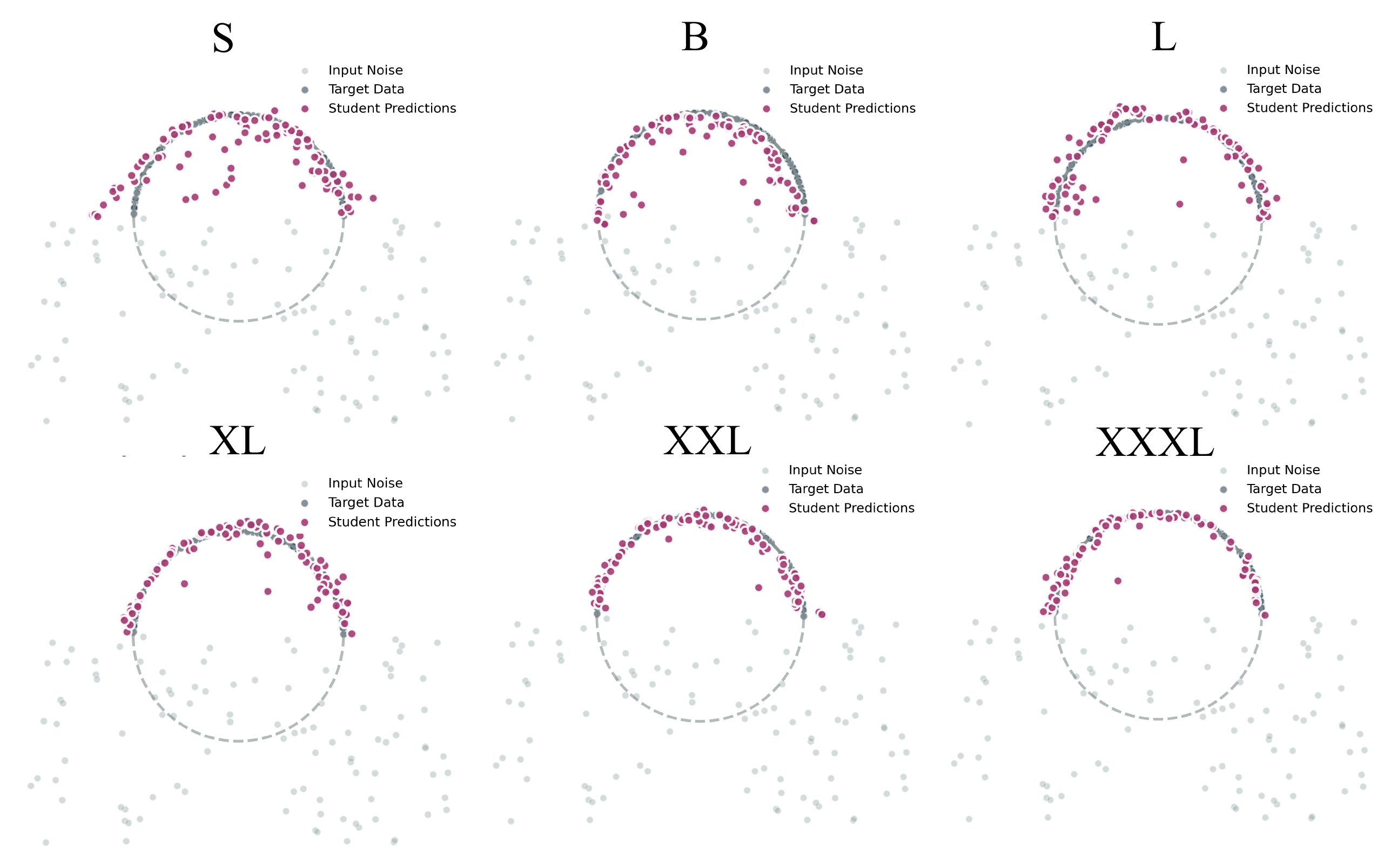}
    \caption{Qualitative comparison of one-step student models trained via Trajectory Distillation across various model sizes (S to XXXL).}
    \label{fig:toy_exp_TD}
\end{figure}

\begin{figure}
    \centering
    \includegraphics[width=\linewidth]{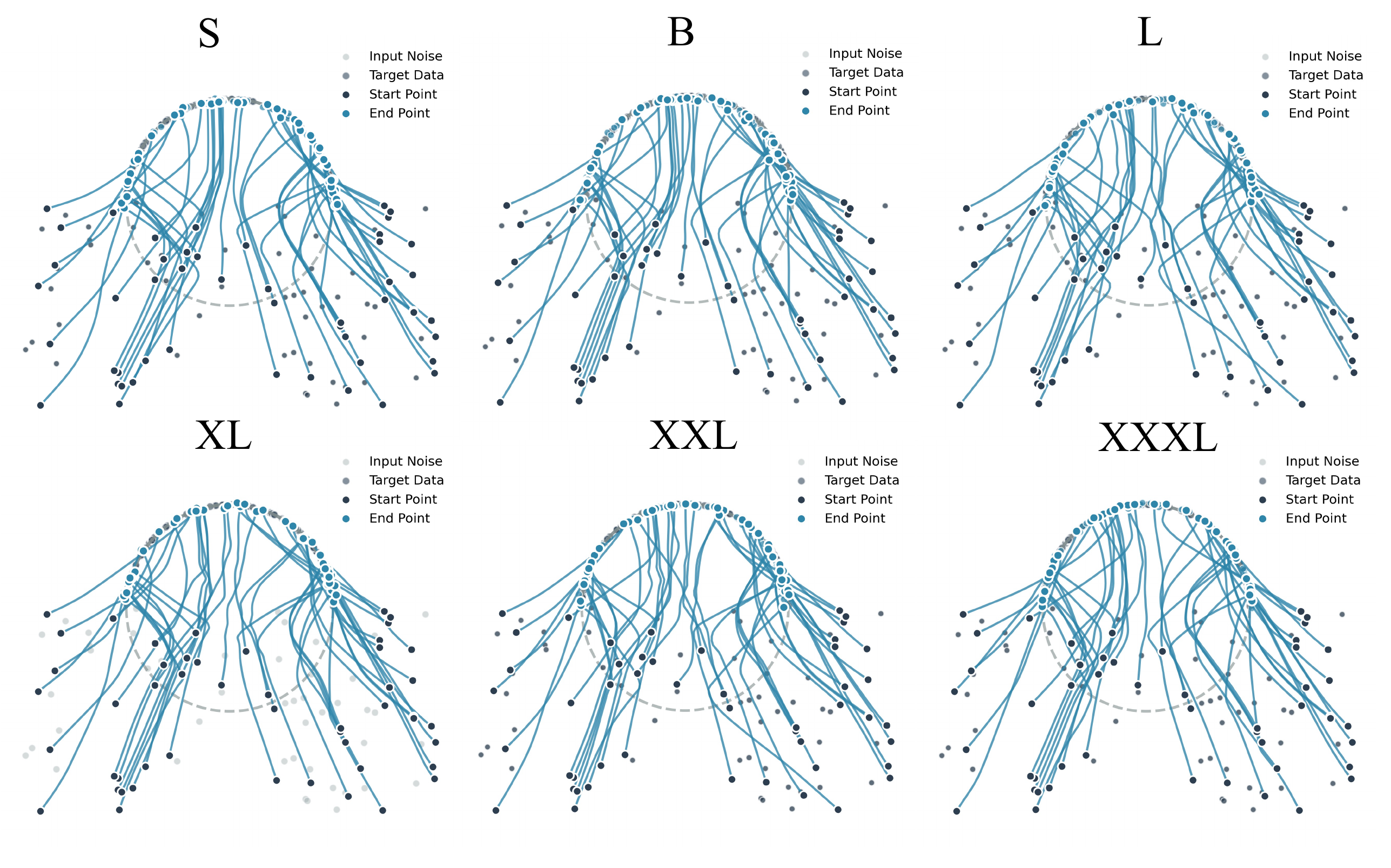}
    \caption{Qualitative comparison of 50-step flow matching teachers across various model sizes (S to XXXL).}
    \label{fig:toy_exp_teacher}
\end{figure}

\section*{Evaluation with FLOPs}
In this section we assess the acceleration achieved by the distilled model in terms of FLOPs
(floating‑point operations).  Experiments on ImageNet with a resolution of $256\times256$
are reported in Table~\ref{tab:image_net_256}.  Our method compresses the inference
procedure of the teacher model (originally requiring 250 diffusion steps) into a single
step, while preserving performance that is comparable to the multi‑step baseline.
To quantify the speed‑up, we compute the FLOPs reduction shown in Table~\ref{tab:flops_evaluation}.
Our method reduces the computational cost from $29.98$ GFLOPs (teacher) to $0.43$ GFLOPs (student), corresponding to a factor of
$\frac{29.98}{0.43}\approx 69.72$ . Thus, on ImageNet with $256\times256$ inputs, our approach significantly reduces the inference cost without sacrificing accuracy.

\begin{table}[h]
    \centering
    \caption{FLOPs and parameter count evaluation results of SiT-XL/2 at 256×256 input resolution (batch size = 1).}
    \setlength{\tabcolsep}{8pt}
    \renewcommand{\arraystretch}{1.3}
    \begin{tabular}{lcccc}  
        \toprule
        & & HD(ours) & Teacher & Teacher \\
        & NFE & 1 & 50 & 250 \\  
        \midrule
        \multirow{2}{*}{DiT} 
            & FLOPs (G)          & 0.12   & 5.93    & 29.67 \\
            & Params (M)         & 676.76 & 676.76  & 676.76 \\
        \midrule
        \multirow{2}{*}{VAE} 
            & FLOPs (G)          & 310.62 & 310.62  & 310.62 \\
            & Params (M)         & 83.65  & 83.65   & 83.65 \\
        \midrule
        \multirow{2}{*}{Total} 
            & FLOPs (T)          & 0.43   & 6.24    & 29.98 \\
            & Params (M)         & 760.41 & 760.41  & 760.41 \\
        \bottomrule
    \end{tabular}
    \label{tab:flops_evaluation}
\end{table}

\section*{Qualitative Results on ImageNet $256 \times 256$}

We provide additional qualitative results of the comparison between our HD method and multi-step teacher on ImageNet $256 \times 256$ as illustrated in Figure \ref{fig:qualitative_imagenet}. The results demonstrate that our method achieves generation quality in a single step that rivals the 250-step teacher model.

\begin{figure*}[ht]

    \centering
    \includegraphics[width=0.8\linewidth]{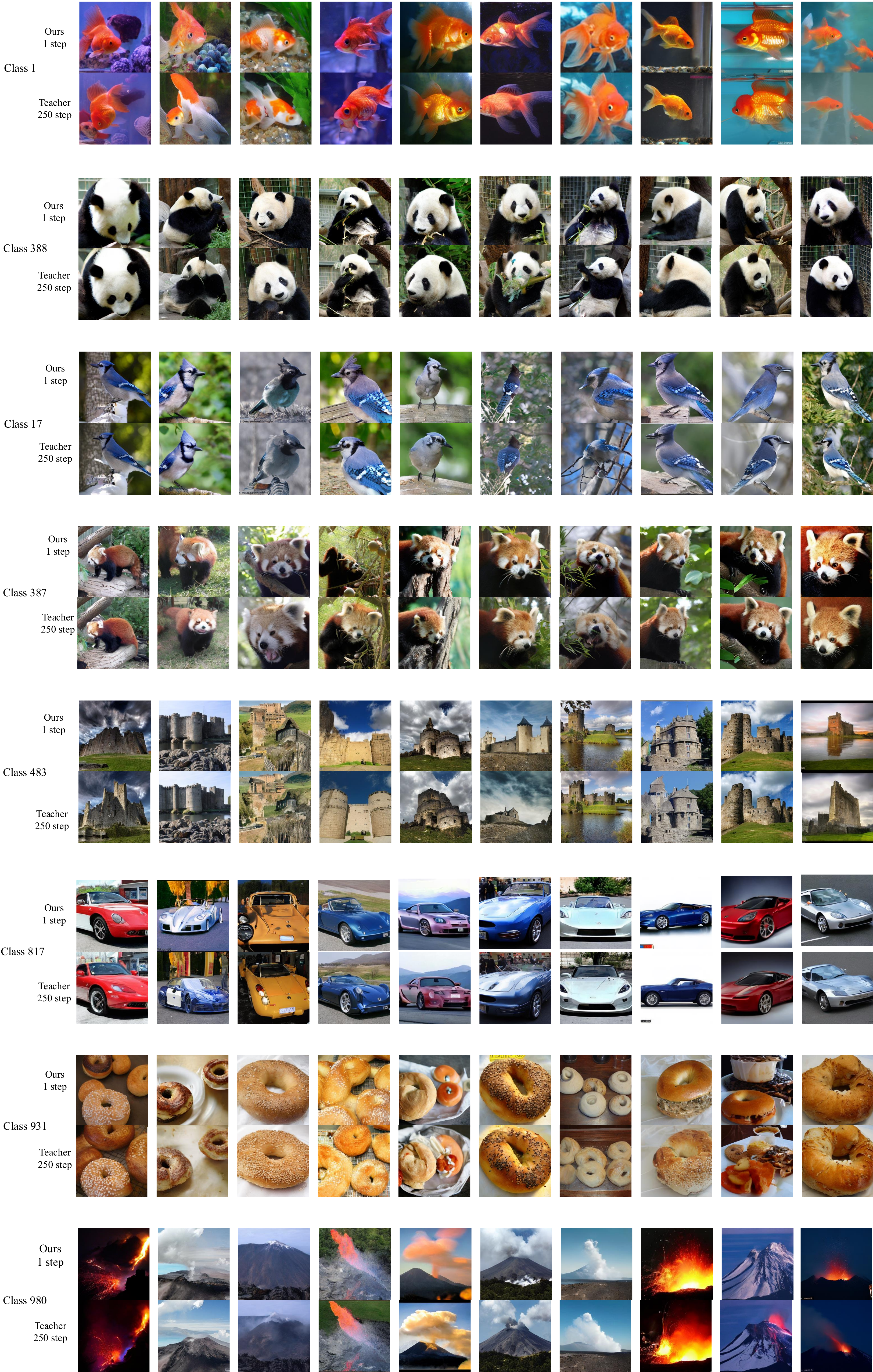}
    \caption{Qualitative comparison of our 1-step student of HD
     method and 250-step SiT teacher on ImageNet $256\times256$.}
    \label{fig:qualitative_imagenet}
\end{figure*}

\end{document}